\newif\ifarxiv 
\arxivtrue

\newif\ifreview 
\reviewfalse

\newif\iffinal 
\finaltrue

\ifarxiv
\documentclass{article}
\else
\ifreview
\documentclass[dvipsnames,format=sigconf,anonymous=true,review=true]{acmart}
\else
\documentclass[dvipsnames,format=sigconf]{acmart}
\fi
\fi

\usepackage[utf8]{inputenc}

\ifarxiv
\usepackage[a4paper,left=2cm,right=2cm,top=2cm,bottom=2cm]{geometry}
\usepackage[onehalfspacing]{setspace}
\usepackage[numbers,longnamesfirst]{natbib}
\usepackage{amsthm,amsthm,amssymb}
\usepackage{authblk}
\fi

\usepackage{amsmath,mathtools} 
\usepackage{dsfont}

\usepackage{booktabs} 
\usepackage{xspace}
\usepackage{graphicx}
\usepackage{xspace}
\usepackage{url}
\usepackage{enumerate}

\usepackage{caption} 
\usepackage{subcaption}

\usepackage[algo2e,ruled,vlined,linesnumbered]{algorithm2e} 
\SetAlgoSkip{}

\usepackage{balance} 

\allowdisplaybreaks[3] 

\usepackage{tikz}
\usetikzlibrary{shapes,arrows}
\usetikzlibrary{decorations}
\usetikzlibrary{plotmarks}
\usetikzlibrary{calc}
\usetikzlibrary{mindmap}
\usetikzlibrary{shadows}
\usetikzlibrary{backgrounds}
\usetikzlibrary{patterns}
\usetikzlibrary{shapes.symbols}

\usepackage{pgfplots}
\usepackage{pgfplotstable}
\usepgfplotslibrary{statistics}
\pgfplotsset{compat=1.18}

\ifarxiv
\newtheorem{theorem}             {Theorem}
\newtheorem{lemma}      [theorem]{Lemma}
\newtheorem{corollary}  [theorem]{Corollary}

\newtheorem{definition} [theorem]{Definition}

\fi

\usepackage{tabularx,colortbl}

\makeatletter
\g@addto@macro\bfseries{\boldmath}
\makeatother

\newcommand{\ie}{i.\,e.\xspace}

\newcommand{\eg}{e.\,g.\xspace}

\newcommand{\Z}{\mathbb{Z}}
\newcommand{\N}{\mathbb{N}}
\newcommand{\R}{\mathbb{R}}

\DeclareMathOperator{\Unif}{Unif}                         
\DeclareMathOperator{\Bin}{Bin}                           

\newcommand*{\bigO}{\mathcal{O}}

\newcommand{\prob}[1]{\Pr\left(#1\right)}                 

\newcommand{\hv}{\ensuremath{\mathrm{HV}}\xspace}         
\newcommand{\deltahv}{\ensuremath{\mathrm{HV}^+}\xspace}  

\newcommand{\nadir}{\ensuremath{\mathrm{nad}}\xspace}      
\newcommand{\ideal}{\ensuremath{\min}\xspace}              

\newcommand{\ones}[1]{|#1|_1}                             
\newcommand{\zeroes}[1]{|#1|_0}                           

\newcommand{\TRAP}{\textsc{Trap}\xspace}                           
\newcommand{\TWOMAX}{\textsc{TwoMax}\xspace}                       

\newcommand{\OMM}{\textsc{OMM}\xspace}                             

\newcommand{\OJZJfull}{\textsc{OneJumpZeroJump}\xspace}            

\newcommand{\OTZT}{\textsc{OTZT}\xspace}                           
\newcommand{\OTZTfull}{\textsc{OneTrapZeroTrap}\xspace}            

\newcommand{\ZOmonotone}{0/1-monotone\xspace}                             

\newcommand{\cdist}{\ensuremath{\textsc{cDist}}\xspace}   

\newcommand{\nsga}{NSGA\nobreakdash-II\xspace}
\newcommand{\nsgaiii}{NSGA\nobreakdash-III\xspace}
\newcommand{\refer}{\mathcal{R}}

\ifarxiv\else 
\fi

\newcommand{\toggleplot}[1]{{\textcolor{red}{Plots removed to increase compilation speed. Use command $\backslash$toggleplot in preamble to reinsert them.}}}

\iffinal
\newcommand{\andre}[1]{\textcolor{green!50!black}{}}
\newcommand{\dirk}[1]{\textcolor{purple}{}}
\newcommand{\cuong}[1]{\textcolor{orange}{}}
\newcommand{\newedit}[1]{\textcolor{black}{#1}}
\newcommand{\neweditx}[1]{\textcolor{black}{#1}}
\newcommand*{\todo}[1]{\textcolor{red}{}}
\else
\newcommand{\andre}[1]{\textcolor{green!50!black}{[(Andre) #1]}}
\newcommand{\dirk}[1]{\textcolor{purple}{[(Dirk) #1]}}
\newcommand{\cuong}[1]{\textcolor{orange}{[(Cuong) #1]}}
\newcommand{\newedit}[1]{\textcolor{black}{#1}}
\newcommand{\neweditx}[1]{\textcolor{blue}{#1}}
\newcommand*{\todo}[1]{\textcolor{red}{\textrm{(TODO: #1)}}}
\fi

\ifarxiv
\else

\copyrightyear{2024}
\acmYear{2024}
\setcopyright{acmlicensed}\acmConference[GECCO '24]{Genetic and Evolutionary Computation Conference}{July 14--18, 2024}{Melbourne, VIC, Australia}
\acmBooktitle{Genetic and Evolutionary Computation Conference (GECCO '24), July 14--18, 2024, Melbourne, VIC, Australia}
\acmDOI{10.1145/3638529.3654177}
\acmISBN{979-8-4007-0494-9/24/07}
\fi

\ifarxiv
\else

\ccsdesc[500]{Theory of computation~Theory of randomized search heuristics}
\ccsdesc[500]{Applied computing~Multi-criterion optimization and decision-making}\fi

\begin{document}

\ifarxiv
\author{Duc-Cuong~Dang}
\author{Andre~Opris}
\author{Dirk~Sudholt}
\affil{University of Passau, Passau, Germany}
\date{}
\fi

\title{\neweditx{Illustrating} the Efficiency of Popular Evolutionary Multi-Objective Algorithms Using Runtime Analysis}

\ifarxiv
\maketitle
\fi

\begin{abstract}



Runtime analysis has recently been applied to popular evolutionary multi-objective (EMO) algorithms like NSGA-II in order to establish a rigorous theoretical foundation. However, most analyses showed that these algorithms have the same performance guarantee as the simple (G)SEMO algorithm. To our knowledge, there are no runtime analyses showing an advantage of a popular EMO algorithm over the simple algorithm for deterministic problems.

We propose such a problem and use it to showcase the superiority of popular EMO algorithms over (G)SEMO:
\OTZTfull is a straightforward generalization of the well-known \TRAP function to two objectives. We prove that, while GSEMO requires at least
$n^{n}$ expected fitness evaluations to optimise \OTZTfull, popular EMO algorithms
NSGA-II, NSGA-III and SMS-EMOA, all
enhanced with a mild diversity mechanism of avoiding genotype duplication,
only require $O(n\log{n})$ 
expected fitness evaluations. 
Our analysis reveals the importance of the key components in each of these sophisticated algorithms 
and contributes to a better understanding of their capabilities.
\end{abstract}

\section{Introduction}\label{sec:intro}

Runtime analysis is a technique for estimating the performance of evolutionary algorithms with rigorous mathematical arguments. 
It bounds the expected time an evolutionary algorithm needs to achieve some goal, such as finding a global optimum, a good approximation, or cover the Pareto front in multi-objective scenarios. These results provide a solid theoretical foundation, help understand the effect of parameters and algorithmic design choices on performance, and inform the search for improved algorithm designs. Runtime analysis was used to identify phase transitions between efficient and inefficient running times for parameters of common selection operators~\cite{Lehre2010a}, the offspring population size in comma strategies~\cite{Rowe2013} or the mutation rate for difficult monotone pseudo-Boolean functions~\cite{Doerr2012c,LenglerS18}. 
Runtime analysis has inspired novel algorithm designs with proven performance gains such as amplifying the success probability of variation operators in the $(1+(\lambda,\lambda))$~GA~\cite{Doerr2015}, choosing mutation rates from heavy-tailed distributions to escape from local optima~\cite{Doerr2017-fastGA} or choosing parameters adaptively throughout the run~\cite{Doerr2019opl,LehreQ22,QinL22}.
For a long time research was restricted to simple algorithms like the $(1+1)$~EA in single-objective optimisation and the simplest possible evolutionary multi-objective algorithms (EMOAs), SEMO and GSEMO~\cite{Laumanns2004,Giel2010}. The latter create one new offspring in every step through selecting a parent uniformly from the population and applying either 1-bit flips (SEMO) or standard bit mutation (GSEMO) and then remove all dominated search points from the population. 

Recent breakthrough results showed that analysing state-of-the-art EMOAs is possible. \neweditx{\citet{ZhengLuiDoerrAAAI22}} provided the first runtime analysis for the NSGA-II algorithm~\cite{Deb2002} as the best cited and most widely used EMOA to date. The key aspect of their analysis \neweditx{\cite{ZhengLuiDoerrAAAI22,ZhengD2023}} was to show that, when the population is chosen large enough, fitness vectors for all non-dominated search points survive to the next generation. This research spawned a number of follow-on papers studying different design aspects of NSGA-II such as the computation of the crowding distance~\cite{Zheng2022}, improving the parent selection~\cite{Bian2022PPSN} or applying heavy-tailed mutations~\cite{DoerrQu22,DoerrQ2023}. 
Recent work showed that crossover can be highly beneficial and speed up the search for the Pareto front~\cite{Dang2023,Bian2022PPSN,Doerr2023Qu,Dang2024}. Runtime analyses was also used to capture and quantify the inefficiency of NSGA\nobreakdash-II 
\neweditx{on the 3-objective extension ($3$-\OMM) of the \textsc{OneMinMax} benchmark function}~\cite{Doerr2023}. 
%
\neweditx{Regarding NSGA-III~\cite{Deb2014} that uses reference points to encourage diversity in many-objective optimisation, \citet{WiethegerD23} showed that on $3$-\OMM if the number of reference points and the population is chosen large enough 
then the function can be optimised efficiently. 
A more general analysis of NSGA-III for many-objective problems will appear in the current conference \cite{Opris2024}.} 
%
SMS-EMOA \cite{Beume2007}, another popular algorithm with non-dominated sorting, was analysed \neweditx{by~\citet{Bian2023}} to show the benefit of a modification to its survival selection, \neweditx{and again in~\cite{Zheng2024} for a many-objective multi-modal function}.

Despite this rapid progress, some important open questions remain. So far, positive results for NSGA-II only show that the same performance guarantees apply that were already known for the much simpler (G)SEMO algorithms. Furthermore, all analyses of NSGA-II and NSGA-III so far only considered search points in the first layer of the population, that is, the non-dominated solutions. Thus, we are lacking a good understanding of when and how keeping dominated solutions in the population is helpful, and why the greater sophistication in NSGA-II over (G)SEMO is justified.
A first step towards answering these questions was made in~\cite{Dang2023}, where it was shown that NSGA-II is more robust on a noisy benchmark problem than (G)SEMO. However, no deterministic optimisation scenario is known where NSGA-II provably outperforms (G)SEMO. 






\textbf{Our contribution:} 
In this work we propose a (deterministic) benchmark function on which more sophisticated algorithms like NSGA-II, NSGA-III and the hypervolume-based SMS-EMOA~\cite{Beume2007} all outperform the simple (G)SEMO algorithm. 
The function is \OTZTfull (\OTZT for short) which is a straightforward generalisation of 
the well-known \TRAP function to multi objectives. 
We prove the inefficiency of simple EMO algorithms GSEMO and SEMO on this 
function due to the fact that they automatically remove dominated solutions 
from the population. On the other hand, popular algorithms that use non-dominated 
sorting to rank solutions like NSGA-II, NSGA-III and SMS-EMOA, when equipped with a mild diversity 
mechanism that avoids genotype duplication, can optimise the function in expected 
polynomial time. These results highlight the benefits of the non-dominated sorting,
as GSEMO and SEMO lack this feature, and the reason for keeping dominated solutions 
while solving EMO problems.

While having copies of solutions is required  
for non-elitist algorithms \cite{Corus2018x}, we think this is not essential
for the listed elitist EMO algorithms\footnote{\neweditx{Papers
\cite{Beume2007,Deb2002,Deb2014} seldom discuss this point 
since the functions used for benchmarking have continuous decision space. But 
popular libraries for EMO (\eg PyMOO \cite{Deb2020}) often provide 
the option to remove duplicates when working with discrete 
optimisation.}}. The mechanism is known to have little or no effectiveness 
in multi-modal optimisation in single objective optimisation \cite{Friedrich2009,Oliveto2015}, and 
as far as we know the only example where this mild mechanism is shown
to be beneficial is \cite{Dang2016a}. Here we show that when combining 
\neweditx{the avoidance of genotype duplication} 
with 
other components of EMO algorithms which themselves tend to improve the population 
diversity, the mechanism is beneficial. 
Our results can also be easy generalised to stronger 
diversity mechanisms, \ie avoiding phenotype duplication. Note that the inefficient 
GSEMO and SEMO algorithms already prevent fitness duplication. 

This is the first runtime analysis of EMO algorithms where 
the survival of the second layer in non-dominated sorting is crucial for the efficiency 
of optimisation. Our work is also the first to rigorously illustrate that popular 
EMO algorithms can outperform (G)SEMO on a deterministic function,
as previously this was only shown for noisy fitness functions \cite{Dang2023a}
or for the opposite statement \cite{Doerr2023}. 
%

\section{Preliminaries}\label{sec:prelim}

%
The sets of integer and natural numbers are $\Z$ and $\N$ respectively.
%
For $n \in \N$, define $[n] := \{1,\dots,n\}$ and
$[n]_0 \coloneqq [n] \cup \{0\}$.
The natural logarithm is denoted $\ln(\cdot)$ and that of base-2 is $\log(\cdot)$.
%
%
For a bit string $x:=(x_1,\dots,x_n)\in\{0,1\}^n$, we use $\ones{x}$ to denote
its number of $1$-bits, \ie $\ones{x}=\sum\nolimits_{i=1}^{n}x_i$,
and similarly $\zeroes{x}$ to denotes its number of zeroes,
\ie $\zeroes{x}=\sum\nolimits_{i=1}^{n}(1-x_i)=n-\ones{x}$.
We use standard asymptotic notation with symbols $O, \Omega, o$~\cite{Cormen2009}.

\begin{definition}
Consider a $d$-objective function $f\colon\neweditx{\{0,1\}}^n\rightarrow \Z^{d}$:
\begin{itemize}
\item For $x, y \in \{0, 1\}^n$, we say $x$ \emph{weakly dominates} $y$
      written as $x \succeq y$ (or $y \preceq x$) if $f_i(x) \geq f_i(y)$ for all $i\in[d]$;
      $x$ \emph{dominates} $y$ written as $x \succ y$ (or $y \prec x$)
      if one inequality is strict.
\item A set of points which covers all possible fitness values not
      dominated by any other points in $f$ is called \emph{Pareto front}
      \newedit{or a \emph{Pareto-optimal set} of $f$}. 
      A single point from the Pareto front is called \emph{Pareto optimal}.
\end{itemize}
\end{definition}

Note that both
the weakly-dominance and dominance relations are \emph{transitive},
\eg $x \succ y \wedge y \succ z$ implies $x \succ z$. 
\neweditx{For a discrete function $f$ as defined above,} 
we say an algorithm $\mathcal{A}$ has optimised $f$ if it have produced 
(or output) a Pareto-optimal set of $f$. 

The $\TRAP(x):=\sum\nolimits_{i=1}^{n} x_i + (n+1)\prod_{i=1}^{n} (1-x_i)$ function was
analysed in the seminal paper \cite{Droste2002} to illustrate that 
the general upper bound $O(n^n)$ on the expected running time of the ($1$+$1$)~EA 
is tight. It is therefore the most difficult pseudo-Boolean function for 
elitist EAs. 
It is a \emph{function of unitation}, \neweditx{\ie} $f(x)$ is a function of $|x|_1$.
To generalise \TRAP to two objectives, it suffices to consider the same 
function but applying to both $|x|_1$ and $|x|_0$:
\begin{align*}
\OTZTfull(x):= \left(\sum\nolimits_{i=1}^{n} x_i + (n+1)\prod_{i=1}^{n} (1-x_i),\right.\\
               \left.\sum\nolimits_{i=1}^{n} (1-x_i) + (n+1)\prod_{i=1}^{n} x_i \right).
\end{align*}
This function is denoted as \OTZT for short, and Figure~\ref{fig:otzt-20} illustrates 
\newedit{it} for $n=20$. Its unique Pareto-optimal set is $F:=\{0^n, 1^n\}$, with 
$\OTZT(0^n)=(n+1,n)$ and $\OTZT(1^n)=(n,n+1)$. \neweditx{The function can be generalised 
to a class of functions by swapping the roles of $1$ and $0$ in some bit positions 
(\eg see \cite{Dang2024} for a similar generalisation), and all our runtime results 
still hold for the whole class.} 
 
An existing function with a similar structure is \OJZJfull \cite{Doerr2021} with its
gap length parameter extended\footnote{Paper \cite{Doerr2021} only defines \OJZJfull 
with the gap lengths below $n/2$ because the structure of the function drastically 
changes otherwise.}
to $n$, however the fitness values of Pareto-optimal points are $(2n,n)$ and $(n,2n)$. 
The two Pareto-optimal points of \OTZT dominate the remaining points of search 
space while the latter points cannot strictly dominate each other. 
These are put formally as follows.

\begin{lemma}\label{lem:otzt-dominance}
On $\OTZTfull$, if $F=\{0^n, 1^n\}$, then: 
\begin{itemize}
\item $\forall x \in F, \forall y \in \{0,1\}^n \setminus F\colon x \succ y$.
\item \newedit{$\forall x, y \in F\colon \neg (x \succ y) \wedge \neg (y \succ x)$.}
\item $\forall x, y \in \{0,1\}^n \setminus F\colon \neg (x \succ y) \wedge \neg (y \succ x)$.
\end{itemize}
\end{lemma}

\begin{figure}[ht]\ifarxiv\centering\fi
\begin{tikzpicture}[x=1em,y=1em,scale=0.7,transform shape]
    \foreach \i in {0,...,22} {
        \draw [very thin,gray] (\i,0) -- (\i,22); 
    }
    \foreach \i in {0,...,22} {
        \draw [very thin,gray] (0,\i) -- (22,\i);
    }
    \node[below] at (11,0) {$f_1(x)=\sum\nolimits_{i}^{n}x_i + (n+1)\prod_{i=1}^{n}(1-x_i)$};    
    \node[above,rotate=90]  at (0,11) {$f_2(x)=\sum\nolimits_{i}^{n}(1-x_i) + (n+1)\prod_{i=1}^{n}x_i$};    
    
    \foreach \i in {1,...,19} {
        \node[circle,fill=blue,scale=0.5] at (\i,20 - \i) {};
    }
    \node[circle,fill=red,scale=0.5] at (21,20) {};
    \node[circle,fill=red,scale=0.5] at (20,21) {};
\end{tikzpicture}
\caption{\OTZTfull with $n=20$.}
\label{fig:otzt-20}
\end{figure}

In order to optimise \OTZT efficiently, it is required that the solutions with 
the maximum number of $1$s and $0$s explored so far are not lost. This is captured 
by the following definition. 
\begin{definition}\label{def:0/1-monotone}
Consider algorithm $\mathcal{A}$ optimising a $d$-objective function 
$f\colon\{0,1\}^n \rightarrow \Z^d$ by evolving a population $P_t$ of search points 
at each of its iteration $t\geq 0$. 
$\mathcal{A}$ is called \emph{\ZOmonotone} on $f$ if both of the following statements 
hold: 
\begin{itemize}
\item $\exists x \in P_{t+1}\colon \ones{x} \geq \max\{\ones{z} \mid z \in P_{t}\}$,
\item $\exists y \in P_{t+1}\colon \zeroes{y} \geq \max\{\zeroes{z} \mid z \in P_{t}\}$.
\end{itemize}
\end{definition}


\section{Inefficiency of (G)SEMO algorithm}\label{sec:gsemo}

The GSEMO and SEMO algorithms follow the same algorithmic framework as shown 
in Algorithm~\ref{alg:gsemo}. Starting from one randomly generated solution, in each 
generation a new offspring solution $y$ is created by copying and mutating a parent
solution $p$ selected uniformly at random from the current population $P_t$. In case
of GSEMO, the standard \emph{bitwise mutation operator} is used, that is, each bit copied 
from the parent has the same and independent probability $1/n$ being flipped where $n$ is 
the length of the string. SEMO uses the so-called \emph{local mutation} in which one bit 
selected uniformly at random is flipped. If $y$ is not dominated by any solutions of 
$P_t$ then it is added to $P_t$, then the next population $P_{t+1}$ is formed by removing 
all those weakly dominated by $y$ from $P_t$. 
The population size
$\lvert{P_t}\rvert$ is unrestricted for GSEMO and SEMO. 
\neweditx{Here and hereafter, we do not specify a stopping criterion since 
we are only interested in the time to reach some target population.}

\begin{algorithm2e}[ht]
	Initialise $P_0:=\{x\}$ where $x \sim \Unif(\{0,1\}^n)$\;
	\For{$t:= 0 \to \infty$}{
		Sample $p \sim \Unif(P_t)$\;
		Create $y$ by copying then mutating $p$ using
          the standard bitwise (GSEMO)
          or local (SEMO) mutation\; 
		\If{$\not\exists x \in P_t\colon x\succ y$}{
			Create the next population $P_{t+1} := P_t \cup \{y\}$\;
			Remove all $x \in P_{t+1}$ where $y\succeq x$\;
		}
	}
	\caption{(G)SEMO Algorithm on $\{0,1\}^n$ \cite{Laumanns2004}.}
	\label{alg:gsemo}
\end{algorithm2e}

The expected optimisation time of (G)SEMO algorithms on \OTZT is essentially identical to that of the
($1$+$1$)~EA and RLS on \TRAP in the single objective setting. The inefficiency of 
(G)SEMO is caused by a combination of multiple factors. The algorithms 
create only one offspring at each generation, they have no control over the population 
size and they have a strict policy that dominated solutions (even for weak dominance)
must be removed.

\begin{theorem}\label{thm:gsemo}
The expected number of fitness evaluations required by SEMO to optimise \OTZTfull
is infinite while that of GSEMO is dominated by a geometric distribution with 
parameter $n^{-n}$. In other words, GSEMO requires at least $n^n$ expected fitness
evaluations to optimise \OTZTfull. 
\end{theorem}
\begin{proof}
Both algorithms start out with a single solution and then evolve their population by 
creating one offspring solution in each iteration. Therefore, the two solutions 
of the Pareto set $F$ are created at different iterations. Assume that $0^n$ is 
created first at some iteration $t\geq 0$. Then, by Lemma~\ref{lem:otzt-dominance}, the 
population collapses to size $1$ until the other Pareto-optimal point $1^n$ 
is created. The latter event requires flipping all the bits, and this is impossible 
for SEMO since the single-bit flip mutation is used. Thus, the expected running time for SEMO
is infinite. For GSEMO, the probability of flipping all bits is $1/n^n$ and the 
expected waiting time is $n^n$. The same reasoning applies when $1^n$ is created first. 
\end{proof}

\section{Analysis of NSGA-II}\label{sec:nsga}


\nsga \cite{Deb2002,NSGAIICode2011} and \nsgaiii~\cite{Deb2014} 
using mutations as the variation operator are summarised in Algorithm~\ref{alg:nsga-ii}. 
Both algorithms use the so-called $(\mu+\mu)$ elitist EA scheme, in which starting
from a randomly initialised population of $\mu$ solutions, in each generation 
new $\mu$ offspring solutions (the population $P_t$) are generated then the 
$2\mu$ solutions (the population $R_t = P_t \cup Q_t$) compete for survival 
in the next population $P_{t+1}$ which is again has size of $\mu$. 
\neweditx{For simplicity we do not consider crossover, 
but our results of polynomial expected running time can be easily 
generalised for the version of the algorithms with any constant crossover probability $p_c \in [0, 1)$.}

We use the most basic choices for the parent selection and mutation operators
to generate the offspring: each offspring solution is created independently
from each other by first picking a parent uniformly at random (with replacement) 
then applying a mutation operator to the copied parent. Like SEMO and GSEMO, we consider 
the standard bitwise mutation and local mutation operators. 

%

In survival selection,
the parent and offspring populations $P_t$ and $Q_t$ are joined into $R_t$,
and then partitioned into
layers 
$F^1_{t+1},F^2_{t+1},\dots$ by the \emph{non-dominated sorting algorithm}~\cite{Deb2002}.
The layer $F^1_{t+1}$ consists of all non-dominated points,
and $F^i_{t+1}$ for $i>1$ only contains points that are dominated by
those from the previous layers $F^1_{t+1},\ldots,F^{i-1}_{t+1}$. The algorithm
then identifies a critical layer $F_t^{i^*}$ such that 
    $\sum\nolimits_{i=1}^{i^*-1} \lvert{F_{t+1}^i}\rvert < \mu$ 
    and $\sum\nolimits_{i=1}^{i^*} \lvert{F_{t+1}^i}\rvert \geq \mu$. 
Solutions from layers $F^1_{t+1},\ldots,F^{i^*-1}_{t+1}$ are then added \neweditx{to}
the next population $P_{t+1}$. The remaining $r:=\mu-\sum\nolimits_{i=1}^{i^*-1} \lvert{F_{t+1}^i}\rvert$
slots of $P_{t+1}$ (if $r>0$) are then taken by solutions from $F_{t+1}^{i^*}$,
and the criterion used for selecting those solutions is where NSGA-II and NSGA-\neweditx{III}
differ (line~\ref{alg:nsga-ii-critical-layer}): NSGA-II uses the crowding distance while NSGA-III uses the distance to 
a predefined set of reference rays after a normalisation procedure. 

The selection of 
solutions in $F_{t+1}^{i^*}$ for NSGA-III is given in the 
next section. For NSGA-II it works as follows. Let $M:=(x_1,x_2,\dots,x_{|M|})$ 
be a multi-set of search points. The crowding distance $\cdist(x_i,M)$ of $x_i$ 
with respect to $M$ is computed as follows. At first sort $M$ as 
$M=(x_{k_1},\dots,x_{k_{\vert{M}\vert}})$ with respect to 
each objective $k \in [d]$
separately. Then
$\cdist(x_i, M)
	:= \sum\nolimits_{k=1}^{d} \cdist_{k}(x_i, M)$, 
	 where
\begin{align}
	\cdist_{k}(x_{k_i}, M)
	&\!:= \! \begin{cases}
		\infty\; & \text{if } i \in \{1, |M|\},\\
		\frac{f_k\left(x_{k_{i-1}}\right) - f_k\left(x_{k_{i+1}}\right)}{f_k\left(x_{k_1}\right) - f_k\left(x_{k_{M}}\right)} & \text{otherwise.}
	\end{cases}\!\!\!\!\label{eq:crowd-dist-eachdim}
\end{align}
The first and last ranked individuals are always
assigned an infinite crowding distance. The remaining individuals
are then assigned the differences between the values of $f_k$ of
those ranked directly above and below the search point and normalized
by the difference between $f_k$ of the first and last ranked. 
NSGA-II then takes $r$ solutions from $F_t^{i^*}$ with the largest computed 
crowding distances from $F_t^{i^*}$ to complete $P_{t+1}$ where ties are broken 
uniformly at random.

So far we have described the original NSGA-II, referred as the \emph{vanilla} 
version of the algorithm and hereafter we use the same term for other algorithms
to refer to their original versions without any enhancement. 
When enhanced with the diversity mechanism of
avoiding genotype duplication (line~\ref{alg:nsga-ii-avoid-duplicate}), 
we only assume: When the solutions in $Q_t$
are merged with $P_t$ to create $R_t$, first duplicated solutions in $Q_t$ are 
reduced to single copies then only those from $Q_t$ that does not exist in $P_t$ 
yet are put into $R_t$ along with $P_t$. So, even though always $\mu$ offspring 
are created, the size of $R_t$ is not necessarily $2\mu$. This requirement is 
minimal in the sense that we do not require the initial population $P_0$ 
to be free of duplicates.

\begin{algorithm2e}[h]
	Initialise $P_0 \sim \Unif( (\{0,1\}^n)^{\mu})$\;
	\For{$t:= 0$ to $\infty$}{
		Initialise $Q_t:=\emptyset$\;
		\For{$i=1$ to $\mu$}{
			Sample $p$ from $P_t$ uniformly at random \label{line:nsga-iii:selection}\;
			Create $x$ by mutating a copy of $p$\;
			Update $Q_t:=Q_t \cup \{x\}$\;
		}
		Set $R_t := P_t \cup Q_t$. \newedit{In case of avoiding genotype duplication: first reduce $Q_t$
        to a proper set, then only $x\in Q_t$ where $x \notin P_t$ are put into $R_t$ along with $P_t$}\label{alg:nsga-ii-avoid-duplicate}\;
		Partition $R_t$ into layers $F^1_{\newedit{t+1}},F^2_{\newedit{t+1}},\ldots$ 
        \newedit{using the non-dominated sorting algorithm}~\cite{Deb2002}\;
            Compute the critical layer $i^* \geq 1$ such that $\sum\nolimits_{i=1}^{i^*-1} \lvert{F_{t+1}^i}\rvert < \mu$ and $\sum\nolimits_{i=1}^{i^*} \lvert{F_{t+1}^i}\rvert \geq \mu$\;
            Set $Y_t := \bigcup_{i=1}^{i^*-1} F_{t+1}^i$\label{alg:nsga-ii:Y_t}\;
            Select a multiset $\tilde{F}_{t+1}^{i^*} \subset F_{t+1}^{i^*}$ of individuals such that $\lvert{Y_t \cup \tilde{F}_{t+1}^{i^*}}\rvert = \mu$: 
                use crowding distance for NSGA-II 
                and distances to reference points (Algorithm~\ref{alg:ref-points}) for NSGA-III\label{alg:nsga-ii-critical-layer}\;
		Create the next population $P_{t+1} := Y_t \cup \tilde{F}_{t+1}^{i^*}$\;
	}
	\caption{
    NSGA-II~\cite{Deb2002} and NSGA-III~\cite{Deb2014} on $\{0,1\}^n$,
    \newedit{with or without avoiding genotype duplication.}}
	\label{alg:nsga-ii}
\end{algorithm2e}


\begin{lemma}\label{lem:nsgaii-progress}
NSGA-II 
    with $\mu\geq \neweditx{4}$,  
    avoiding genotype duplication 
    \newedit{and with $P_0\cap \{0^n,1^n\}=\emptyset$}
is \ZOmonotone on \OTZTfull. 
\end{lemma}
\begin{proof}
At any generation $t\geq 0$, let $i:=\max\{\ones{z} \mid z\in F^{i^*}_{t+1}\}$ and 
$j:=\max\{\zeroes{z}\mid z\in F^{i^*}_{t+1}\}$. 
Consider the following cases: 

If $R_t \cap F=\emptyset$: no Pareto-optimal solutions have been created yet. 
By Lemma~\ref{lem:otzt-dominance} none of the solutions in $R_t$ can dominate each 
other thus they all belong to the first layer $F^1_{t+1}$, in other words $F^1_{t+1}=F^{i^*}_{t+1}=R_t$. 
The ranking of these solutions 
then uses the crowding distance calculation. \newedit{In this calculation, if the sorted 
population according to the first objective is $(S_1[1],\dots,S_1[\mu])$, then it holds 
that $f(S_1[1])=(i, n-i)$ and $f(S_1[\mu])=(n-j,j)$ and that $S_1[1]$ and $S_2[\mu]$
are assigned infinite crowding distances.}
\newedit{Similarly, in population $(S_2[1],\dots,S_2[\mu])$ which is sorted according 
to the second objective, $S_2[1]$ and $S_2[\mu]$ are assigned infinite crowding distances with
$f(S_2[1])=(n-j, j)$ and $f(S_2[\mu])=(i,n-i)$. Thus the solutions from the set 
$\{S_1[1], S_2[\mu], S_2[1], S_2[\mu]\}$ are the only ones with infinite crowding distances} 
\neweditx{and, consequently, they are top ranked.} 
\neweditx{Particularly, in any three solutions of that set} 
there are always at least two with fitness $(i, n-i)$ and $(n-j, j)$. 
Because $\mu\geq \neweditx{4>3}$ slots are available in $P_{t+1}$ \newedit{and duplicates are 
rejected, one copy of} the top \neweditx{three} solutions is guaranteed to be kept in $P_{t+1}$. 
Among them, we have \neweditx{a solution $x$} with $\ones{x}=i \geq \max\{\ones{z} \mid z\in P_{t}\}$ 
and the last inequality follows from that $P_t \subseteq R_t=F^{i^*}_{t+1}$. 
Likewise,
\neweditx{we keep some $y$ with} $\zeroes{y}=j \geq \max\{\zeroes{z}\mid z\in P_{t}\}$. 

If $R_t \cap F\neq\emptyset$: Pareto-optimal solutions have been created. If both
$0^n$ and $1^n$ are in $R_t$ then the results are obvious because these optimal solutions 
are kept in $P_{t+1}$ and they have a maximum numbers of $1$s and $0$s, respectively.
If only one of them is in $R_t$, \newedit{we first note that they can only have one copy
in $R_t$ because given that $P_0 \cap \{0^n,1^n\}=\emptyset$ then the Pareto-optimal
solution can only be generated in $Q_{t'}$ at some generation $t'\leq t$ and the diversity 
mechanism reduces its (possible) duplicates into a single copy before \neweditx{integration} into 
$R_t'$.}
By Lemma~\ref{lem:otzt-dominance} that Pareto-optimal solution, denoted by $z$,  
is in the first non-dominated layer $F^1_{t+1}$ while the rest of $R_t$ are in the next 
non-dominated layer $F^2_{t+1}=F^{i^*}_{t+1}$. 
The solutions in $F^2_{t+1}$ then compete for the remaining $\mu':=\mu-|F^1_{t+1}| 
\geq \neweditx{3}$ slots in $P_{t+1}$. Repeating the same argument as in the previous case 
but for layer $F^2_{t+1}$ and the available $\mu'$ slots implies that two solutions $x$ and 
$y$ with the maximum numbers of $1$s and $0$s respectively in $F^2_t$ survive to $P_{t+1}$. 
The proof is completed by noting that the surviving solutions $x,y,z$ contain the 
maximum numbers of $1$s and $0$s of $R_t$ since $R_t=F^1_{t+1}\cup F^2_{t+1}$ and that
$R_t\supseteq P_t$.
\end{proof}

The lemma guarantees the progress of the algorithm on \OTZT. 
The requirement $\mu \ge \neweditx{4}$ stems from the fact \neweditx{that in the worst case we need 
one slot to store the first Pareto-optimal found in the top layer and three other 
slots to store solutions with infinite crowding distances in the next layer. 
This is because there can be four such solutions (\ie $S_1[1],S_2[\mu],S_2[1],S_2[\mu]$) 
and any three of those will cover the two extreme points for that layer.}
\neweditx{This reasoning holds regardless of which sorting algorithm is 
used for computing the crowding distances.}

\begin{theorem}\label{thm:nsgaii}
NSGA-II with $\mu\geq \neweditx{4}$ and using \newedit{the standard} bitwise mutation or 
\newedit{the local mutation} 
as the variation operator and the diversity mechanism of avoiding genotype duplication 
optimises \OTZTfull in $\bigO(\mu n\log{n})$ \neweditx{expected} fitness evaluations. 
\end{theorem}
\begin{proof}
By Lemma~\ref{lem:nsgaii-progress}, the progress in both numbers of $1$s and $0$s is 
not lost, owing to the choice of the population size and the diversity mechanism. It remains
to estimate the expected number of generations until both $0^n$ and $1^n$ of $F$ are present
in the population. 

At generation $t$, define $i:=\max\{\ones{z} \mid z\in P_{t}\}$. In one offspring production,
to create an individual with more $1$s, it suffices to pick one individual $x$ 
with $|x|_1=i$ then flipping one of its $0$s while keeping the rest unchanged, this occurs with 
probability $\frac{1}{\mu}\cdot {n-i \choose 1} \frac{1}{n}\left(1-\frac{1}{n}\right)^{n-1}
\geq \frac{n-i}{e\mu n}=:s_i$ for \newedit{the standard} bitwise mutation, and respectively $\frac{n-i}{\mu n}>s_i$ 
for \newedit{the local mutation}. 
With $\mu$ offspring productions, the probability is amplified
to $1-(1-s_i)^{\mu}\geq \frac{\mu s_i}{\mu s_i + 1}$ by Lemma~10 in~\cite{Badkobeh2015}. 
Therefore, starting from the initial population with $i_0:=\max\{\ones{z} \mid z\in P_{0}\}$,
the expected number of generations to create the $1^n$ solution is no more than:
\begin{align*}
\sum\nolimits_{i=i_0}^{n-1} \left(1 + \frac{1}{\mu s_i}\right)
  \leq \sum\nolimits_{i=0}^{n-1} \left(1 + \frac{en}{n-i}\right)
  = \bigO(n\log{n}).
\end{align*}

Similarly, starting from  $j_0:=\max\{\zeroes{z} \mid z\in P_{0}\}$, the expected number of 
generations to create the $0^n$ solution is also $\bigO(n\log n)$. This is also the overall
expected number of generations to optimise the function. The result in terms on fitness 
evaluations follows by noting that $\mu$ evaluations are required in each generation.
\end{proof}

The result implies that NSGA-II avoiding genotype duplication in its population 
can optimise $\OTZT$ in $O(n\log{n})$ \newedit{evaluations} using only a constant 
population size, even with the use of local mutation instead of the standard bitwise mutation. 
The efficiency of the algorithm on the function is partially attributed to the use of 
crowding distances to classify the solutions at the critical layer, and particularly to assignment
of infinite distances to the extreme points. 
However, if no diversity mechanism is used, then even for any population size $\mu=o(\sqrt{n})$ 
the expected running time of both NSGA-II and NSGA-III with bitwise mutation is 
\newedit{already} $\Omega(n^{n})$.

\begin{theorem}\label{thm:nsgaii-no-diversity}
The vanilla NSGA-II \newedit{and NSGA-III} (without avoidance of genotype duplicates) with $\mu=o(\sqrt{n})$ 
and using uniform parent selection and bitwise mutation requires \newedit{$\Omega(n^{n})$} 
 expected fitness evaluations to optimise \OTZTfull. 
\end{theorem}

\begin{proof}
\newedit{We only need to show the result for NSGA-II because both algorithms follow the same
algorithmic scheme and only differ in the selection of the individuals in the critical
layer. The proof below holds regardless of how this selection is done.}
By a union bound, the initial population does not contain both the two $0^n$ and $1^n$
Pareto-optimal solutions with probability $1-2 \cdot \mu \cdot 2^{-n}=1-o(1)$. Assume that
event occurs, still $0^n$ and $1^n$ can be simultaneously created in one generation of 
NSGA-II. 

\newedit{We now estimate a lower bound on the probability of not creating both $0^n$ and 
$1^n$ of $F$ in one generation. For this, we consider the most ideal setting for the opposite 
situation of creating both optima at the same time. Specifically, we assume that when 
a solution is selected it always has a probability $2/n$ of producing a solution in $F$ 
as the offspring. This is because at least one bit must be flipped to create these solutions 
from a non Pareto-optimal one, \ie this occurs with at most probability $1/n$ by bitwise 
mutation, and the factor $2$ comes from the union bound on the cardinality of $F$. Furthermore, 
we optimistically assume that when more than one solution in $F$ is created, then the Pareto 
front is fully covered. Let $X$ be the number of offspring solutions of $F$ in $Q_t$, then 
under the above assumptions $X\sim\Bin(\mu,2/n)$, 
the probability of discovering the Pareto-front is $p_1:=\prob{X\geq 1} = 1 - \prob{X=0}$ 
while that of fully covering the front is $p_2:=\prob{X \geq 2} = 1 - \prob{X=0} - \prob{X=1}$. 
It follows from Lemma~10 in~\cite{Badkobeh2015} and $\mu=o(\sqrt{n})$ that
\begin{align*}
p_1 &= 1 - \left(1 - \frac{2}{n}\right)^{\mu} 
    \geq \frac{2\mu/n}{2\mu/n+1} 
    = \frac{2\mu/n}{o(1)+1}
    = \Omega\left(\frac{\mu}{n}\right).
\end{align*}}

\newedit{By Bernoulli's inequality we have 
$\left(1-\frac{2}{n}\right)^{\mu-1} \geq 1 - \frac{2(\mu-1)}{n}$, thus
\begin{align*}
p_2 &=    1 - \left(1 - \frac{2}{n}\right)^{\mu} - {\mu \choose 1}\left(\frac{2}{n}\right)\left(1 - \frac{2}{n}\right)^{\mu-1} \\
    &=    1 - \left(1 - \frac{2}{n} + \frac{2\mu}{n}\right)\left(1 - \frac{2}{n}\right)^{\mu-1} \\
    &\leq 1 - \left(1 + \frac{2(\mu-1)}{n}\right)\left(1 - \frac{2(\mu-1)}{n}\right)
    = \frac{4(\mu-1)^2}{n^2} = \bigO\left(\frac{\mu^2}{n^2}\right).
\end{align*}
Therefore, the probability of fully covering the front conditioned on it has been discovered
is $p_2/p_1 = O(\mu/n) = o(1)$. Thus when the front is discovered in a generation, the probability
that is not fully covered by the offspring is at least $1-o(1)$.}

Conditioned on the event that only a solution $x\in F$ is created first at generation $t$, 
we now estimate the probability for $x$ to take over the whole population before the creation 
of the other solution $y \in F\setminus \{x\}$ in future generations. 
\newedit{The argument is similar to Theorem~2~in~\cite{Friedrich2009} but here adapted to 
the generation of $\mu$ offspring solutions instead of only $1$ in that paper.} 
An offspring production is 
called \emph{good}, denoted as event $G$, if a copy of $x$ is produced, and this happens 
with probability of 
$\prob{G}\geq \frac{1}{\mu}\cdot \left(1-\frac{1}{n}\right)^n\geq \frac{1}{4\mu}$
for $n\geq 2$. 
A production is called \emph{bad}, denoted as $B$, if $y$ is produced and this occurs with 
probability of 
$\prob{B}\leq \frac{1}{n}$ since at least one bit is required to be flipped from any 
selected parent. 
A sufficient condition for $x$ to entirely take over the population is that $2(\mu-1)$
good productions occur before the first bad one. The factor $2$ here takes into account
that the $(\mu-1)$-th good productions may occur at the beginning of some generation $t'>t$ 
thus the rest of the productions in $t'$ still requires to be not bad.
The probability of 
the event is then at least:
\begin{align*}
\ifarxiv
\left(\frac{\prob{G}}{\prob{G}+\prob{B}}\right)^{2(\mu-1)}&
\fi 
   \geq \left(\frac{1/(4\mu)}{1/(4\mu)+1/n}\right)^{2(\mu-1)}\\
    &= \left(1 - \frac{1/n}{1/(4\mu)+1/n}\right)^{2(\mu-1)}
    \geq \left(1 - \frac{4\mu}{n}\right)^{2(\mu-1)}\\
    &\geq 1 - \frac{8\mu(\mu-1)}{n} = 1-o(1)\neweditx{,}
\end{align*}
\newedit{where the last inequality follows from Bernoulli's inequality.}

After this series of events, that happen with probability $1-\newedit{o(1)}$, the only way to 
create $y$ is to flip all the bits of a copy of $x$ and this occurs with probability $1/n^n$. 
By the law of total probability 
\newedit{the expected number of fitness evaluations to optimise the function}
is at least \newedit{$(1-o(1))\cdot n^n = \Omega(n^{n})$}.
\end{proof}

The above analysis is similar to that of \TWOMAX in \cite{Friedrich2009}, however we have
the ($\mu$+$\mu$)-selection scheme here instead of the steady-state scheme, and this gives 
some advantages to the algorithm. On the other hand, the situation is not exactly the same
since we are in a multiple-objective setting, thus non-dominated solutions can be created 
and can enter the population. 


\section{Analysis of NSGA-III}

NSGA-III~\cite{Deb2014} uses the same elitist $(\mu+\mu)$ framework as NSGA-II 
(\ie Algorithm~\ref{alg:nsga-ii}). However, the algorithm uses distances between normalised 
fitness vectors and a set of predefined reference rays and employs 
a \emph{niching} procedure to choose the remaining $r=\mu-\sum\nolimits_{i=1}^{i^*-1} |F_{t+1}^{i}|$ 
solutions of $F_{t+1}^{i^*}$ to complete $P_{t+1}$. 

Fitness vectors of already selected solutions $Y_t$ (set in line~\ref{alg:nsga-ii:Y_t} 
in~Algorithm~\ref{alg:nsga-ii}) and those in the critical layer $F_{t+1}^{i^*}$ are normalised. 
We consider the improved normalisation procedure from~\cite{Blank2019} 
which treats the boundary cases more effectively. The detailed procedure can be found in 
that paper, but the following description and properties are sufficient for our purpose. 
For a $d$-objective function $f\colon \{0,1\}^n\rightarrow \Z^{d}$, the normalised fitness 
vector $f^{n}(x):=(f_1^n(x),\dots,f_d^n(x))$ of search point $x$ is computed with:
\begin{align}
f_j^n(x)
    =\frac{f_j(x)-z_j^{\ideal}}{z_j^{\nadir}-z_j^{\ideal}}\label{eq:nsgaiii-normalize}
\end{align}
for each $j\in[d]$. The points $z^{\ideal}:=(z^{\ideal}_1, \dots, z^{\ideal}_d)$ and 
$z^{\nadir}:=(z^{\nadir}_1, \dots, z^{\nadir}_d)$ are referred to as the \emph{ideal} and 
\emph{nadir} points, respectively, of the objective space. Particularly, $z_j^{\ideal}$
is set to the minimum value in objective $j$ either from the current population, 
or from all search points that the algorithm has explored so far. In this paper, to simplify
the analysis we consider the former way of setting $z_j^{\ideal}$ based on the current population. 
The computation of the nadir point requires treating some boundary cases with care, however 
the procedure in \cite{Blank2019} guarantees for each $j\in[d]$ that
\begin{align}
(z_j^{\nadir} \geq \varepsilon_{\nadir}) 
\wedge (z_j^{\nadir} > z_j^{\ideal}) 
\wedge (z_j^{\nadir} \leq z_j^{\max})\label{eq:nsgaiii-nadir-prop}
\end{align}
where $z_j^{\max}$ is defined similar to $z_j^{\ideal}$ but for the maximum objective 
value in objective $j$ 
and $\varepsilon_{\nadir}$ is a small positive threshold. 

\begin{algorithm2e}[h]
Compute the normalized $f^n(x)$ 
for each $x\in Y_t \cup F_{t+1}^{i^*}$~\cite{Blank2019}\;
Associate each $x\in Y_t \cup F_{t+1}^{i^*}$ to its reference point $\mathrm{rp}(x)\in\refer$ 
based on the smallest distance to the reference rays\;
For each $r \in \refer$, initialise $\rho_r:=|\{x\in Y_t \mid \mathrm{rp}(x)=r\}|$\;
Initialise $\tilde{F}_{t+1}^{i^*}:=\emptyset$ and $R':=\refer$\;
\While{true}{
Determine $r_{\min} \in R'$ such that $\rho_{r_{\min}}$ is minimal, and ties are broken randomly\;
Determine $x_{r_{\min}} \in F_{t+1}^{i^*} \setminus \tilde{F}_{t+1}^{i^*}$ which is associated with $r_{\min}$ 
  and minimises the distance between $f^n(x_{r_{\min}})$ and the ray of $r_{\min}$, and ties are broken randomly\;
\If{$x_{r_{\min}}$ exists}{
    Update $\tilde{F}_{t+1}^{i^*} := \tilde{F}_{t+1}^{i^*} \cup \{x_{r_{\min}}\}$\;
    Update $\rho_{r_{\min}} := \rho_{r_{\min}} + 1$\;
    \If{$\lvert{Y_t}\rvert + \lvert{\tilde{F}_{t+1}^{i^*}}\rvert = \mu$}
        {\Return{$\tilde{F}_{t+1}^{i^*}$}
        }
    }
    \Else{
        Update $R':=R' \setminus \{r\}$\;
    }
}
\caption{\textsc{Niching} procedure~\cite{Deb2014} 
    to compute $\tilde{F}_{t+1}^{i^*}$ from $F_{t+1}^{i^*}$ based on 
    a set $\refer$ of reference points.}
\label{alg:ref-points}
\end{algorithm2e}

After normalization, each of the solutions of $Y_t \cup F_{t+1}^{i^*}$ is associated with
a reference point from a predefined set $\refer$. This association is based on the corresponding 
nearest reference ray to the normalised fitness vector of the solution. Here, the reference ray 
of a reference point is the ray that is originated from the origin and goes through the reference 
point. The niching procedure \cite{Deb2014} is then applied to identify the solutions 
$\tilde{F}_{t+1}^{i^*}$ from $F_{t+1}^{i^*}$ to complete $P_{t+1}$ along with $Y_t$. This is 
summarised in Algorithm~\ref{alg:ref-points}. 
This simplified description\footnote{The original paper~\cite{Deb2014} also makes a difference
between the cases $\rho_{r_{\min}}=0$ and $\rho_{r_{\min}}\geq 1$ but has a footnote 
explaining the variant~\cite{WiethegerD23} as described here.} 
of the procedure is due to \cite{WiethegerD23}. 
First, the niche count $\rho_r$ for each reference point $r\in\refer$ is initialised as the 
number of solutions in $Y_t$ already associated to $r$. The procedure then iterates through
the reference points, always selecting the one with the smallest niche count so far (ties are 
broken randomly), to fill $\tilde{F}_{t+1}^{i^*}$ with solutions from $F_{t+1}^{i^*}$ associated 
with the selected reference point and update the niche count. Priority is given to those with 
the smallest distance to the reference ray (ties are broken randomly).

There are different ways to define the reference points. The original paper~\cite{Deb2014} 
suggests the method from Das~and~Denis~\cite{Das1998} to define this set with
a parameter $p\in\N$:
\begin{align}
\refer := \left\{\left(\frac{a_1}{p},\dots,\frac{a_d}{p}\right)
                 \mid (a_1,\dots,a_d)\in \N^{d}_0, \sum\nolimits_{i=1}^{d} a_i=p\right\}\neweditx{.}\label{eq:nsgaiii-refpts-das-dennis}
\end{align} 
However, our result for \OTZTfull below only makes a minimal requirement for $\refer$ that
it should \neweditx{contain} the unit vectors.
%

\begin{lemma}\label{lem:nsgaiii-progress}
NSGA-III
    with $\refer\supseteq \{(0,1),(1,0)\}$, 
    $\mu \geq 2|\refer|$, 
    $z_j^{\min}$ computed from the current population, 
    and avoiding genotype duplication 
is \ZOmonotone on \OTZTfull if $P_0\cap \{0^n,1^n\}=\emptyset$. 
\end{lemma}
\begin{proof}
Like in the proof of Lemma~\ref{lem:nsgaii-progress}, let $i:=\max\{\ones{z} \mid z\in F^{i^*}_{t+1}\}$ and 
$j:=\max\{\zeroes{z}\mid z\in F^{i^*}_{t+1}\}$ for a generation $t\geq 0$, we consider 
the same case separation. 

If $R_t \cap F=\emptyset$: No Pareto-optimal solutions have been created yet.
By Lemma~\ref{lem:otzt-dominance} all the solutions in $R_t$ are then incomparable 
and belong to layer $F^1_{t+1}$, thus $R_t=F^1_{t+1}=F^{i^*}_{t+1}$. The selection 
of these solutions to survive in $P_{t+1}$ then 
use Algorithm~\ref{alg:ref-points}. By the definition of $i$ and $j$, we note that
    $z^{\ideal}_1 = \min\{\ones{z} \mid z\in R_{t}\}=n-j$ 
    and $z^{\ideal}_2 = \min\{\zeroes{z} \mid z\in R_{t}\}=n-i$. The normalization
according to Equation~\eqref{eq:nsgaiii-normalize} on all solutions $x$ with $f(x)=(i,n-i)$
and $y$ with $f(y)=(n-j,j)$ gives $f^n(x)=(a,0)$ and $f^n\neweditx{(y)}=(0,b)$ 
where
\begin{align*}
a &= \frac{i - (n-j)}{z^{\nadir}_1 - (n-j)}
   = \frac{i+j - n}{z^{\nadir}_1 - (n-j)} \geq 0\\
b &= \frac{j - (n-i)}{z^{\nadir}_2 - (n-i)}
   = \frac{i+j - n}{z^{\nadir}_2 - (n-i)} \geq 0.
\end{align*}
The non-negativity is due to \eqref{eq:nsgaiii-nadir-prop}. 
Note that $a$ and $b$ can only be zero at the same time since they have the 
same numerator $i+j-n$. In this case, all solutions in $F^{i^*}_{t+1}$ have the same fitness 
vector and $\mu$ of them chosen randomly survive to $P_{t+1}$. It is then obvious 
that the maximum numbers of $1$s and $0$s remain unchanged between $P_t$ and $P_{t+1}$.
Otherwise if $a>0$ and $b>0$, then the distances between $f^n(x)=(a,0)$ and 
reference ray of $(1,0)$ and between $f^n(x)=(0,b)$ and reference ray of $(0,1)$
are both $0$, which is the smallest possible distance. Thus all points $x$ are 
associated to $(1,0)$ and all points $y$ are associated to $(0,1)$. All other points 
$z$ of $R_t$ with $f(z)\neq f(x)$ and $f(z)\neq f(y)$ have that 
    $f_1^n(z)=\frac{\ones{z}-(n-j)}{z^{\nadir}_1 - (n-j)}>0$ (because $f(z)\neq f(y)$) 
    and $f_2^n(z)=\frac{\zeroes{z}-(n-i)}{z^{\nadir}_2 - (n-i)}>0$ (because $f(z)\neq f(x)$).
Thus they all have larger distances to the reference rays of $(1,0)$ 
and $(0,1)$ compared to $x$ and $y$. 
Because $\mu\geq 2|\refer|>|\refer|$ and Algorithm~\ref{alg:ref-points}
always gives priority to reference point with the smallest 
 \neweditx{niche count 
so far, every} reference point can keep at least one solution associated 
to it with preference to those with the smallest distance. Therefore, at least one $x$ 
solution and one $y$ solution survives to $P_{t+1}$ with 
    $\ones{x}=i \geq \max\{\ones{z} \mid z\in P_{t}\}$ and 
    $\zeroes{y}=j \geq \max\{\zeroes{z}\mid z\in P_{t}\}$
since $P_t\subseteq R_t$.

If $R_t \cap F\neq\emptyset$: Pareto-optimal solutions have been created. 
Like in Lemma~\ref{lem:nsgaii-progress} the case that both $0^n$ and $1^n$ are contained in~$R_t$ is trivial 
as we have a population with the maximum possible numbers of $1$s and $0$s already and both search points will be kept.
If only one of them is in $R_t$, we then argue similarly to the previous lemma
that it can have only one copy based on the initial condition $P_0 \cap \{0^n,1^n\}=\emptyset$.   
By Lemma~\ref{lem:otzt-dominance} that one copy of Pareto-optimal solution,
denoted by $z$, is in the first non-dominated layer $F^1_{t+1}$ which is also $Y_t$, 
while the rest of $R_t$ are in the next non-dominated layer $F^2_{t+1}=F^{i^*}_{t+1}$. 
Repeating the same argument as in the previous case but for the layer $F^2_{t+1}$
gives that solutions $x$ with $f(x)=(i,n-i)$ and $y$ with $f(y)=(n-j,j)$ are
associated to reference points $(1,0)$ and $(0,1)$ respectively. 
We now consider the worst case of the niching procedure given in
Algorithm~\ref{alg:ref-points} that the Pareto-optimal solution
$z$ is associated to one of these reference points, say $(1,0)$ 
(the argument is symmetric in the other case), \newedit{and that all 
reference points other than $(1,0)$ and $(0,1)$ also have associated 
solutions which can compete with solutions $x$ and $y$}. 
This initialises $\rho_{(1,0)}$ to $1$ and the main loop of the niching 
procedure can be only run for at most $\mu-|Y_t|=\mu-1$ iterations. After the 
first $|\refer|-1$ iterations, all niche counts are set to $1$ 
and particularly a solution $y$ is included in $\tilde{F}^2_{t+1}$ hence will survive to $P_{t+1}$. 
Afterwards, the order for which the reference points are selected is random.
Nevertheless, if after $|\refer|-1$ further iterations, if $(1,0)$ is
not selected yet then still $\rho_{(1,0)}=1$ while all $\rho_{r}=2$ for all
$r\in \refer\setminus\{(0,1)\}$, then in the next iteration a solution
$x$ is included in $\tilde{F}^2_{t+1}$. So far we have used $2(|\refer|-1)+1$
iterations in our argument, and this is still smaller than the limit of $\mu-1$ 
iterations given $\mu\geq 2|\refer|$. 
We have shown that three solutions $x,y,z$ survives to $P_{t+1}$ and particularly
they carry the maximum numbers of $1$s and $0$s in $R_t$ and these number
are no less than those of $P_t$ because $P_t\subseteq R_t$.
\end{proof}


As soon as the conditions of the lemma are satisfied, NSGA-III cannot lose track
of the progress on the numbers of $1$s and $0$s. Particularly, it then behaves
like NSGA-II in Theorem~\ref{thm:nsgaii}. We then have the following result whose 
 proof follows exactly the same steps at that of Theorem~\ref{thm:nsgaii}, thus
this proof is omitted.

\begin{theorem}\label{thm:nsgaiii}
NSGA-III as described in Lemma~\ref{lem:nsgaiii-progress} 
using the standard bitwise mutation or the local mutation 
optimises \OTZTfull in $\bigO(\mu n \log{n})$
expected fitness evaluations. 
\end{theorem}

We give a corollary to the specific choice of $\refer$ as in Equation \eqref{eq:nsgaiii-refpts-das-dennis}.
This result holds because for any $p\geq 1$, $(p,0)$, $(0,p)$ are always valid 
choices for $a_1, a_2$ in a bi-objective problem, and these result in the unit 
vectors after the normalisation with $p$. 

\begin{corollary}\label{cor:nsgaiii-das-dennis}
NSGA-III with $\refer$ as defined in Equation~\eqref{eq:nsgaiii-refpts-das-dennis}
for $p\geq 1$ and the rest of the parameters as described in Lemma~\ref{lem:nsgaiii-progress} 
and using the standard bitwise mutation or the local mutation 
in $\bigO(\mu n \log{n})$ expected fitness evaluations. 
\end{corollary}

NSGA-III can optimise \OTZT in $O(n\log{n})$ using
a constant number of reference points (only 2 reference points are sufficient) and
with a constant population size. The success on NSGA-III on this function is mainly 
due to a proper set of reference points, and the normalisation procedure with the 
use of the ideal point. This normalisation enables the preference of the extreme
solutions of the dominated layer as this is required to optimise \OTZT. 
%
%
When copies of solutions are allowed, NSGA-III also becomes inefficient for $o(\sqrt{n})$ 
population size as shown in Theorem~\ref{thm:nsgaii-no-diversity}. 

\section{Analysis of SMS-EMOA}\label{sec:sms-emoa}

SMS-EMOA~\cite{Beume2007} is another popular EMO algorithm that uses the non-dominated 
sorting \cite{Deb2002} to rank the solutions. It is described in Algorithm~\ref{alg:sms-emoa}. 
Unlike NSGA-II and NSGA-III, it uses the \emph{steady-state} scheme of ($\mu$+$1$)~EA. 
Furthermore, the algorithm uses the contribution to the hypervolume~\cite{Zitzler1998,ZitzlerT1999} 
of the critical layer, which is also the lowest layer in the ($\mu$+$1$) scheme, to choose 
the worst individual to eject during the survival selection (lines~\ref{alg:sms-emoa-hypervolume} 
and \ref{alg:sms-emoa-eject}). When avoiding genoptype duplication is used, the algorithm 
simply reject the offspring 
if a copy of it already exists in 
the 
population 
(line~\ref{alg:sms-emoa-reject-duplicate}).

The computation of the hypervolume is done with respect to a reference point $h$ 
that the algorithm has to set at the beginning (in line~\ref{alg:sms-emoa-refpts}). 
The hypervolume of a set $P$ of solutions
with respect to a fitness function $f\colon\mathcal{X}\rightarrow \R^{d}$ and the point
$h$ is defined as:
\begin{align*}
\hv(P) := \mathcal{L}\left(\bigcup_{x\in P} \left\{y\in\R^{d}\mid h\preceq y \wedge y \preceq f(x)\right\}\right)
\end{align*}
where $\mathcal{L}$ denotes the Lebesgue measure. 
\neweditx{ 
Let $\deltahv(x,P)$ be the contribution of a point $x\in P$ to $\hv(P)$, then:}
\begin{align*}
\deltahv(x,P) := \hv(P) - \hv(P\setminus \{x\}).
\end{align*}

\begin{algorithm2e}[h]
Initialise $P_0 \sim \Unif( (\{0,1\}^n)^{\mu})$\;
Choose a reference point $h\in\R^{d}$\label{alg:sms-emoa-refpts}\;
\For{$t:= 0$ to $\infty$}{
    Sample $p$ from $P_t$ uniformly at random\;
	Create $x$ by mutating a copy of $p$\;
	Set $R_t := P_t \cup \{x\}$. In case of avoiding genotype duplication: 
    if $x \in P_t$ set $P_{t+1}:=P_t$ and go directly to the next iteration\label{alg:sms-emoa-reject-duplicate}\;
	Partition $R_t$ into layers $F^1_{t+1},F^2_{t+1},\ldots,F^{i^*}_{t+1}$ 
    using the non-dominated sorting algorithm~\cite{Deb2002}\;
    Compute the \neweditx{hypervolume contribution} $\deltahv(y, F^{i^*}_{t+1})$ using $h$ as the reference point 
      for each $y\in F^{i^*}_{t+1}$ 
      and find the set $Y$ of points with the smallest contribution\label{alg:sms-emoa-hypervolume}\;
    Sample $y\sim\Unif(Y)$ and create the next population $P_{t+1} := R_t \setminus\{y\}$\label{alg:sms-emoa-eject}\;
}
\caption{
SMS-EMOA~\cite{Beume2007} on $\{0,1\}^n$,
with or without avoiding genotype duplication.}
\label{alg:sms-emoa}
\end{algorithm2e}

\begin{lemma}\label{lem:sms-emoa-deltahv}
Consider SMS-EMOA 
    using $h \preceq (-(n/2)^2, -(n/2)^2)$ 
      on \OTZTfull with $n\geq 4$. 
If a layer $F:=F^i_{t+1}$ from the immediate population $R_t=P_t\cup\{x'\}$ of SMS-EMOA has three 
non-Pareto-optimal points $x_a, x_b, x_c$ with unique fitness in $F$ such that: 
\begin{itemize}
\item $\OTZT(x_a)=(a,n-a)$, $a=\max\{\ones{x}\mid x\in F^i_t\}$
\item $\OTZT(x_b)=(b,n-b)$, $n-b=\max\{\zeroes{x}\mid x\in F^i_t\}$
\item $\OTZT(x_c)=(c,n-c)$, $b<c<a$\newedit{,}
\end{itemize}
then the contributions these points to the hypervolume of $F^i_t$ satisfy:
\begin{align*}
\left(\deltahv(x_a, F)>\deltahv(x_c, F)\right)  
    \wedge \left(\deltahv(x_b, F)>\deltahv(x_c, F)\right).
\end{align*}
\end{lemma}


\begin{proof}
We have $n-1\geq a>c>b\geq 1$ because the points are non-Pareto optimal. 
The contribution of $x_c$ to the hypervolume is the largest when there is no other 
point than $x_c$ between $x_a$ and $x_b$ because 
any other points may diminish the hypervolume contribution of $c$, but have no effect 
on that of $a$ and $b$.
Formally, 
\begin{align*}
\deltahv(x_c,F) 
    &\leq (c-b)((n-c)-(n-a)) \\
    &\leq (c-1)((n-c)-1) = (c-1)(n-2-(c-1))
\end{align*}
and as a function of $c-1$, this upper bound is maximised when $c-1=\lceil (n-2)/2 
\rceil=\lceil n/2 \rceil-1$, so
\begin{align*}
\deltahv(x_c,F) 
    &\leq (\lceil n/2 \rceil-1)(\lfloor n/2 \rfloor-1)
    < (n/2)^2.
\end{align*}

Because $x_a$ is the unique point with the largest number of $1$s, then its contribution
to the hypervolume is a rectangle of width \neweditx{at least}~1 and height at least $(n/2)^2$, thus $\deltahv(x_a,F) \ge (n/2)^2$.

Similarly, the contribution of $x_b$ as the unique point with the largest number of $0$s and so its hypervolume contribution is the volume of a rectangle of height at least 1 and width at least $(n/2)^2$, thus $\deltahv(x_b,F) \ge (n/2)^2$.
\end{proof}

\begin{lemma}\label{lem:sms-emoa-progress}
SMS-EMOA 
    \neweditx{
    using $h \preceq (-(n/2)^2, -(n/2)^2)$, 
    with $\mu\geq 3$,} 
    and avoiding genotype duplication 
is \ZOmonotone on \OTZTfull for $n\geq 4$
if $P_0\cap \{0^n,1^n\}=\emptyset$. 
\end{lemma}
\begin{proof}
For any generation $t\geq 0$, let $i:=\max\{\ones{z} \mid z\in F^{i^*}_{t+1}\}$ and 
$j:=\max\{\zeroes{z}\mid z\in  F^{i^*}_{t+1}\}$.
Consider the following cases:

If $R_t \cap F=\emptyset$: no Pareto-optimal solutions have been created yet. 
By Lemma~\ref{lem:otzt-dominance} all the solutions in $R_t$ then belong to the first 
layer $F^1_{t+1}$. This means $F^1_{t+1}=F^{i^*}_{t+1}=R_t$ and the ranking of the 
solutions uses the hypervolume contributions.
Consider the case that there exist two solutions $x$ and $x'$ in $F^{i^*}_{t+1}$ with
$f(x)=f(x')$, then their contributions to the hypervolume of $F^{i^*}_{t+1}$ are both 
$0$. However, at least one of those solutions survives to $P_{t+1}$ since the algorithm 
only removes one worst offspring. In other words, the set of fitness vectors of the 
population remains unchanged from $P_t$ to $P_{t+1}$ if there are duplicates of fitness 
in $F^{i^*}_{t+1}$. Otherwise, if there is no duplicates of fitness in $F^{i^*}_{t+1}$, 
consider the solutions $x, y$ with $f(x)=(i, n-i)$, $f(y)=(n-j,j)$ and any other solution
$z$ of $F^{i^*}_{t+1}$, then applying Lemma~\ref{lem:sms-emoa-deltahv} for $x_a=x$, 
$x_b=y$ and $x_c=z$ 
the contributions of $x$ and $y$ to the hypervolume of $F^{i^*}_{t+1}$ are always 
larger than that of any $z$, thus $x,y$ have the highest contributions to the hypervolume
and survive to $P_{t+1}$ given $\mu\geq 3>2$. 
Then we have 
    $\ones{x}=i \geq \max\{\ones{z} \mid z\in P_{t}\}$ and 
    $\zeroes{y}=j \geq \max\{\zeroes{z}\mid z\in P_{t}\}$ since $P_{t}\subseteq R_t=F^{i^*}_{t+1}$. 

If $R_t \cap F\neq\emptyset$: Pareto-optimal solutions have been created. If both
$0^n$ and $1^n$ are in $R_t$ these optimal solutions are kept in $P_{t+1}$ and they have 
the maximum numbers of $1$s and $0$s, respectively.
If only one of them is in $R_t$ then by Lemma~\ref{lem:otzt-dominance} that solution 
is in $F^1_t$ while all other, different solutions in $R_t$ are in the next non-dominated layer $F^2_t$. 
Thus \newedit{$F_{t+1}^1$ contains only one copy of a Pareto-optimal solution, and that solution,
denoted $z$, will survive to $P_{t+1}$}.
The solutions 
in $F^2_{t+1}$ compete for the remaining $\mu':=\mu-|F^1_{t+1}| \geq 2$ slots in $P_{t+1}$
using the contributions to the hypervolume of $F^2_{t+1}$. Repeating the same argument 
as in the previous case but now for $F^{i^*}_{t+1}=F^2_{t+1}$ and the $\mu'$ available slots 
implies that two solutions $x,y$ with the maximum numbers of $1$s and $0$s respectively 
of $F^2_{t+1}$ survive to $P_{t+1}$. The surviving solutions $x,y,z$ contains the maximum 
numbers of $1$s and $0$s of $F^1_{t+1}\cup F^2_{t+1} = R_t$, and these numbers are no less 
than those of $P_t$ since $P_t \subseteq R_t$.
\end{proof}

\begin{theorem}\label{thm:sms-emoa}
SMS-EMOA as described in Lemma~\ref{lem:sms-emoa-progress}
and using the standard bitwise mutation or the local mutation 
optimises \OTZTfull with $n\geq 4$ in $\bigO(\mu n\log{n})$ fitness evaluations. 
\end{theorem}
\begin{proof}
By Lemma~\ref{lem:sms-emoa-progress}, the progress in both numbers of $1$s and $0$s are 
not lost owing to the choice of the population size, the diversity mechanism, and
the choice of the reference point $h$. 
It remains to estimate the expected number 
of generations until both $0^n$ and $1^n$ of $F$ are presented in the population. 

At generation $t$, define $i:=\max\{\ones{z} \mid z\in P_{t}\}$, then similar to the 
argument in the proof of Theorem~\ref{thm:nsgaii}, the probability to increase the number
of $1$s in the next generation is at least $\frac{n-i}{e\mu n}=:s_i$ for both mutation 
operators. Thus, the expected number of generations to create the $1^n$ solution is no more 
than: 
\begin{align*}
\sum\nolimits_{i=i_0}^{n-1} \frac{1}{s_i}
  = \sum\nolimits_{i=0}^{n-1} \frac{e\mu n}{n-i}
  = \bigO(\mu n\log{n}).
\end{align*}
This is also asymptotically the expected number of generations to create the $0^n$ 
solution, and the overall expected number of generations or fitness evaluations to 
optimise the function. 
\end{proof}


Similar to the results of NSGA-II and NSGA-III, this result shows that SMS-EMOA 
that avoids genotype duplication can also optimise $\OTZTfull$ in $O(n\log{n})$ fitness
evaluations. This efficiency is attributed to the proper choice of the reference point
in the calculation of the hypervolume contribution. This allows the preference of the 
extreme points in the critical layer. 
When copies of solutions are allowed in the population, SMS-EMOA becomes inefficient
if its population is too small. The proof argument is similar to the analysis of \TWOMAX
in \cite{Friedrich2009}, thus we will not repeat it here. 

\begin{theorem}\label{thm:sms-emoa-no-diversity}
The vanilla SMS-EMOA requires $\Omega(n^n)$ \neweditx{expected}
fitness evaluations to optimise \OTZTfull \neweditx{if $\mu=o(n/\log{n})$}. 
\end{theorem}

\section{Conclusions}\label{sec:conclude}
\sloppy
We \neweditx{introduced \OTZTfull as an example function on which the simple EMO algorithms 
 (G)SEMO fail}. \neweditx{In contrast,} popular EMO algorithms using non-dominated sorting, including \mbox{NSGA\nobreakdash-II}, 
NSGA-III and SMS-EMOA, can succeed with a very small population size and with a mild 
diversity mechanism of avoiding genotype duplication. These results demonstrated
the benefits of the non-dominated sorting, and other \neweditx{key algorithmic} components, including the use of either crowding distances, reference points or 
hypervolume calculation, in solving EMO problems. We hope that our results can provide 
theoretical foundations for the popularity and wide adoption of these popular algorithms 
in practice.

\ifarxiv
\section*{Acknowledgments}
\else
\begin{acks}
\fi
The authors would like to thank Benjamin Doerr for 
the helpful discussion on the structure of the \OJZJfull function.
\ifarxiv\else
\end{acks}
\fi

\balance
\bibliographystyle{ACM-Reference-Format}
\bibliography{references} 


\end{document}